\RequirePackage[l2tabu,orthodox]{nag}		
\documentclass[reqno]{amsart}
\usepackage[margin=1.5in,bottom=1.25in]{geometry}		

\usepackage{amsmath}		
\usepackage{amssymb}		
\usepackage{amsfonts}		
\usepackage{amsthm}		
\usepackage[foot]{amsaddr}		

\usepackage{mathtools}		

\mathtoolsset{%
}

\usepackage[utf8]{inputenc}		
\usepackage[T1]{fontenc}		

\usepackage[
cal=cm,
]
{mathalfa}

\usepackage{dsfont}		

\usepackage[proportional,tabular,lining,sf,mono=false]{libertine}
\usepackage{acronym}		
\newcommand{\acdef}[1]{\emph{\acl{#1}} \textup{(\acs{#1})}\acused{#1}}		

\usepackage[labelfont={bf,small},labelsep=colon,font=small]{caption}	
\usepackage[dvipsnames,svgnames]{xcolor}		
\colorlet{MyRed}{Crimson!75!Black}
\colorlet{MyGreen}{DarkGreen!80!Black}
\colorlet{MyBlue}{MediumBlue}

\usepackage[]{titlesec}		
\titleformat{name=\section}{}{\thetitle.}{0.8em}{\centering\scshape}
\titleformat{name=\subsection}[runin]{}{\thetitle.}{0.5em}{\bfseries}[.]
\titleformat{name=\subsubsection}[runin]{}{\thetitle.}{0.5em}{\bfseries}[.]

\titleformat{name=\paragraph,numberless}[runin]{}{}{0em}{\bfseries}[]
\titlespacing{\paragraph}{0em}{\medskipamount}{1em}
\titleformat{name=\subparagraph,numberless}[runin]{}{}{0em}{}[.]
\titlespacing{\subparagraph}{0em}{0em}{0.5em}

\usepackage{latexsym}		

\usepackage{pifont}		

\usepackage{subcaption}		
\usepackage{tikz}		
\usetikzlibrary{calc,patterns}		

\usepackage{array}		
\usepackage{booktabs}		
\usepackage[inline,shortlabels]{enumitem}		
\setenumerate{itemsep=\smallskipamount,topsep=\smallskipamount}		

\usepackage[kerning=true]{microtype}		

\usepackage{xspace}		

\usepackage[authoryear,sort&compress]{natbib}		

\bibpunct[, ]{(}{)}{,}{}{,}{,}

\usepackage{hyperref}
\hypersetup{
colorlinks=true,
linktocpage=true,
pdfstartview=FitH,
breaklinks=true,
pdfpagemode=UseNone,
pageanchor=true,
pdfpagemode=UseOutlines,
plainpages=false,
bookmarksnumbered,
bookmarksopen=false,
bookmarksopenlevel=1,
hypertexnames=true,
pdfhighlight=/O,
urlcolor=MyBlue,linkcolor=MyBlue,citecolor=MyGreen,	
pdftitle={},
pdfauthor={},
pdfsubject={},
pdfkeywords={},
pdfcreator={pdfLaTeX},
pdfproducer={LaTeX with hyperref}
}

\usepackage[sort&compress,capitalize,nameinlink]{cleveref}		
\crefname{algocf}{Algorithm}{Algorithms}
\Crefname{algocf}{Algorithm}{Algorithms}
\crefname{assumption}{Assumption}{Assumptions}

\usepackage[ruled]{algorithm2e}

\newtheorem{assumption}{Assumption}

\newtheorem{theorem}{Theorem}[section]
\newtheorem{lem}{Lemma}[section]
\newtheorem{corollary}{Corollary}[section]

\newtheorem{defin}{Definition}[section]
\newtheorem{remark}{Remark}[section]

\newcommand{\para}[1]{\paragraph{#1.}}

\DeclareMathOperator*{\argmin}{arg\,min}		

\newcommand{\blambda}{\boldsymbol{\lambda}}

\DeclareMathOperator{\proj}{proj}	  
\DeclareMathOperator{\dom}{dom}		

\DeclareMathOperator{\reg}{Reg}		
\DeclareMathOperator{\res}{Res}		

\def\real{\mathbb R}

\newcommand{\bd}{{\bf d}}

\newcommand{\by}{{\bf y}}

\newcommand{\bo}{{\bf 0}}
\newcommand{\bx}{{\bf x}}

\newcommand{\bz}{{\bf z}}

\newcommand{\R}{{\mathcal{R}}}

\newcommand{\eg}{{e.g., }}

\begin{document}


\newcommand{\longtitle}{Regret Minimization in Stochastic Non-Convex Learning\\
via a Proximal-Gradient Approach}		
\newcommand{\runtitle}{Regret Minimization in Stochastic Non-Convex Learning}		

\title[\MakeUppercase{\runtitle}]{\MakeUppercase{\longtitle}}		

\author
[N.~Hallak]
{Nadav Hallak$^{\ast,\lowercase{c}}$}
\address{$^{\ast}$\,%
The Technion, 3200003, Haifa, Israel.}
\address{$^{c}$\,Corresponding author.}
\email{ndvhllk@technion.ac.il}

\author
[P.~Mertikopoulos]
{Panayotis Mertikopoulos$^{\diamond,\sharp}$}
\address{$^{\diamond}$\,%
Univ. Grenoble Alpes, CNRS, Inria, LIG, 38000, Grenoble, France.}
\address{$^{\sharp}$\,%
Criteo AI Lab.}
\email{panayotis.mertikopoulos@imag.fr}

\author
[V.~Cevher]
{Volkan Cevher$^{\ddag}$}
\address{$^{\ddag}$\,%
École Polytechnique Fédérale de Lausanne (EPFL).}
\email{volkan.cevher@epfl.ch}


\thanks{%
The work of N.~Hallak was conducted at EPFL, and was supported by the European Research Council (ERC) under the European Union's Horizon 2020 research and innovation programme (grant agreement no 725594 - time-data).
P.~Mertikopoulos is also grateful for financial support by
the French National Research Agency (ANR) under grant no.~ANR\textendash 16\textendash CE33\textendash 0004\textendash 01 (ORACLESS).
V.~Cevher gratefully acknowledges the support of the Swiss National Science Foundation (SNSF) under grant \textnumero\ 200021\textendash 178865/1, the European Research Council (ERC) under the Horizon 2020 research and innovation programme (grant agreement \textnumero\ 725594 - time-data), and 2019 Google Faculty Research Award.
This research was also supported by the COST Action CA16228 ``European Network for Game Theory'' (GAMENET)}

\newacro{LHS}{left-hand side}
\newacro{RHS}{right-hand side}
\newacro{iid}[i.i.d.]{independent and identically distributed}
\newacro{lsc}[l.s.c.]{lower semi-continuous}

\newacro{MAB}{multi-armed bandit}
\newacro{TAP}{traffic assignment problem}
\newacro{ONTAP}[OnTAP]{online traffic assignment problem}
\newacro{SFO}{stochastic first-order oracle}

\begin{abstract}
%
%
%
Motivated by applications in machine learning and operations research, we study regret minimization with stochastic first-order oracle feedback in online constrained, and possibly non-smooth, non-convex problems. In this setting, the minimization of external regret   is beyond reach, so we focus on a local regret measure defined via a proximal-gradient mapping. To achieve no (local) regret in this setting, we develop a prox-grad method based on stochastic first-order feedback, and a simpler method for when access to a perfect first-order oracle is possible. Both methods are min-max order-optimal, and we also establish a bound on the number of prox-grad queries these methods require.  As an important application of our results, we also obtain a link between online and offline non-convex stochastic optimization manifested as a new prox-grad scheme with complexity guarantees matching those obtained via variance reduction techniques.
\end{abstract}

\allowdisplaybreaks		
\acresetall		
\maketitle

\section{Introduction}
\label{sec:1}



First-order methods have proven to be extremely flexible and efficient in online convex optimization:
they enjoy tight performance guarantees
in a wide range of relevant settings  such as convex, strongly convex, composite, etc.,
and they can adapt to different measures of regret
under different oracle feedback assumptions, e.g., perfect/stochastic gradients or bandit feedback.
For example, see \cite{ABRT08,HAK07,hazan_introduction_2016} and \cite{Xia10} for applications to different convex settings,
\cite{BGZ15,CBGLS12}, and \cite{HazSes09} for variant regret measures,
and
\cite{ABRT08,ADX10}, and \cite{BE16,BE17} for a range of feedback assumptions.

On the other hand, many contemporary problems, especially in machine learning, involve highly multi-modal \emph{non-convex} functions.
In this case, the results obtained in the above framework do not \textendash\ in fact, \emph{cannot} \textendash\ apply, and new analytical tools and algorithms are needed.
Nevertheless, and somewhat surprisingly at that, online non-convex optimization problems are not as well explored, and significantly less is known about the performance of first-order methods in this context.

The key difficulties encountered in the online non-convex setting are twofold:
First, the standard regret comparator of a ``best action in hindsight'' (fixed or otherwise) is too ambitious because, in general, even \emph{offline} non-convex optimization problems are intractable.
Second, compared to problems with a convex structure, non-convex problems have no local-to-global guarantees, so the adversary has a near-insurmountable advantage (in analogy to non-convexified/non-randomized optimizers facing an adversarial bandit).
Our paper seeks to address these challenges in a unified way.

\para{Related work}

One approach to treat  online non-convex optimization is to regard the problem as an adversarial \ac{MAB} with a \emph{continuum} of arms.
This approach was pioneered by \cite{BMSS11,Kle04} and \cite{KSU08}, who proposed a range of hierarchical search methods, with and without a doubling trick, that guarantee no regret in problems with a geometry that is amenable to local search such as the hypercube.
\cite{KBTB15} and, more recently, \cite{PML17} and \cite{HMMR20}, took an approach based on a suitable adaptation of the Hedge/EXP3 algorithms to bandits with a continuum of arms and established the method's no-regret properties under relatively mild regularity conditions.
However, in full generality, sampling from continuous Gibbs distributions can be quite challenging, so it is not a-priori clear how to implement these methods without a sampling oracle in place.

Another approach, manifesting in the recent works of \cite{agarwal_learning_2018} and \cite{suggala_online_2019}, is the classical Follow-the-Perturbed-Leader algorithm with access to an \emph{offline non-convex optimization oracle}, which was shown to enjoy a polynomial regret bound.
Simplifying assumptions that render a non-convex problem tractable, were also considered in the literature in more particular cases such as the principal component analysis model; see  \cite{G19} and references therein for additional examples.

Complementing this literature in an orthogonal direction, \cite{hazan_efficient_2017} took a more direct, ``pure-strategy'', approach based on a ``smoothed'' inner-loop / outer-loop version of projected gradient descent.
In this general framework, a straightforward extension of Cover's impossibility result shows that the minimization of standard regret measures is unattainable.
On account of this, \cite{hazan_efficient_2017} considered instead a \emph{local regret} measure based on a sliding evaluation window and a suitable measure of stationarity (as opposed to \emph{optimality}).
When faced with a stream of Lipschitz smooth functions, the algorithm of \cite{hazan_efficient_2017} enjoys a local regret bound that scales with the horizon $T$ of the process and the size $w$ of the sliding window as $O(T/w^{2})$, with projection calls complexity $O(Tw)$;
as a result, sublinear (local) regret \emph{is} possible as long as $w=\omega(1)$.
Importantly, \cite{hazan_efficient_2017} also showed that the local regret bound is unimprovable from a min-max perspective, so the proposed algorithm is optimal in this regard. 
For \emph{unconstrained} problems with stochastic gradient observations, \cite{hazan_efficient_2017} further showed that a suitable variant of their method achieves similar guarantees in expectation. 

\para{Our contributions}

Our goals are twofold:
First, 
we seek to treat online  problems that are potentially \emph{non-smooth}, covering \eg the case of $L^{1}$-regularization.
Second, in line with the above, we also wish to account for problems with \emph{stochastic} oracle feedback, simultaneously with constraints and regularization, thus including  problems subjected to both random and seasonal fluctuations.
To achieve the desiderata, we consider a  general \emph{composite} non-convex online framework in which each loss function encountered consists of a smooth and non-smooth part; this study is the first  to provide methods with theoretical guarantees to address this scenario.
Concisely, our main contributions are
\begin{itemize}
	\item Assuming access to only a stochastic first-order oracle, we introduce a smoothed \emph{prox-grad} method to handle \emph{stochastic, constrained, non-smooth, non-convex}  online optimization problems with tight regret guarantees of $O(T/w^2)$ in expectation and stochastic first-order oracle calls bound of $O(w^3)$.
	This represents a significant step forward relative to the literature, mainly, compared to the online stochastic method proposed by \cite{hazan_efficient_2017}, as the latter can only address the basic \textit{smooth unconstrained} case.
	\item Relaxing the feedback assumptions to  a perfect first-order oracle, we also present a simpler method  that can simultaneously tackle online non-convex optimization problems with both constraints and regularization,  and obtain tight regret guarantees $O(T/w^2)$ with prox-grad calls complexity $O(w^2)$ in the process.
	\item As a by-product, but of an independent interest and contribution of its own, we derive from our methods new schemes for stochastic offline optimization under the online framework assumptions with the best known guarantees,  achievable only via variance reduction techniques (see  \cite{ACDFSW19} and references therein).
\end{itemize}

\section{Problem setup}
\label{sec:2}

\subsection{Statement of the problem and blanket assumptions}

We consider the class of online non-convex, nonsmooth, composite  problems over a finite and discrete time horizon $T\geq 1$ of the form
\begin{equation}
\tag{P}
\label{prob:1}
\min  \{   \ell_{t}(\bx) = f_t (\bx) + g(\bx) : \ \bx\in \real^n \}, \qquad t\in [T],
\end{equation}
where
\begin{enumerate}
	\item   $g:\real^n \rightarrow\real_+\cup\{\infty\}$ is a proper, convex, lower semicontinuous (l.s.c) function.
	\item For any $t\in [T]$, the function $f_t\colon \real^n\rightarrow \real$ is   $L$-smooth ($L>0$ ) over $\dom g$, i.e., 
	\begin{align*}
	&\| \nabla f_t(\bx) - \nabla f_t(\by) \| \leq L  \|\bx - \by\| \qquad \forall\bx,\by\in \dom g, \ \forall t\in [T].
	\end{align*}
	\item There exists $M>0$ such that for any $\bx\in\dom g$ and $t\in [T]$, it holds that $| f_t (\bx) | \leq M $.
\end{enumerate}
Our blanket assumptions are fundamental in the study of online learning, even when the objective function is convex (see e.g., \cite{hazan_introduction_2016}). 
We also note that $f_t$ is  assumed to be $L$-smooth and bounded  only over the domain of $g$, meaning that  if $\dom g$ is bounded, then the assumptions on $f_t$ trivially hold true.


\subsection{Motivating applications}
\label{sec:routing}
Examples of \eqref{prob:1} are ubiquitous in theoretical computer science, operations research, and many other fields where online decision-making is the norm. 
For concreteness, we shortly describe next a few conceptual examples; further details are provided in the supplement.
\begin{itemize}
	\item \textbf{Non-convex games:} 
	A multi-player non-convex game can be modeled by simultaneously optimizing several copies of \eqref{prob:1}, where all share the same function $f_t$, and (un-shared) penalty functions may be utilized to induce stability (e.g., risk aversion) in the choices of each of the players independently; see  \eg  \cite{hazan_efficient_2017,agarwal_learning_2018}.
	
	A particularly interesting instance of a two players  non-convex game in which the feasible set is usually compact, and the objective function is accessible through a stochastic oracle, is the  \textit{generative adversarial network (GAN)} model; GANs were already considered via an online framework by \cite{GLLHK17} and \cite{agarwal_learning_2018} for example.
	
	
	\item \textbf{Online path planning with splittable traffic demands:} 
	The \acl{ONTAP} is a hallmark path planning problem that requires the full capacity of our model, and whose  formulation further applies to learning perfect matchings, multitask bandits, spanning tree exploration, etc. 
	Referring to \cite{BG92} and \cite{SS08} for an introduction to the topic, the key objective in \aclp{TAP} is the optimal allocation of traffic over a given network with variable traffic inflows. 
	The feasibe set here is compact, the cost functions are smooth yet non-convex, and   a sparsity-inducing $L^{1}$ term is typically included to ``robustify'' solutions by minimizing the overall number of paths employed; we provide a fully detailed formulation in the supplement.

\item \textbf{Stochastic (offline) optimization:}
Stochastic optimization,  which follows naturally from online optimization by restricting the  adversarial behavior accordingly, plays a prominent role in modern applications, such as neural networks.
\end{itemize}


\subsection{Local regret minimization}
\label{sec:blocks}

In the online non-convex framework of \eqref{prob:1}, there are two key issues with the standard definition of the regret as $\reg(T) = \max_{\bx \in \dom g} \sum_{t=1}^{T} [\ell_{t}(\bx_{t}) - \ell_{t}(\bx)]$:
First, the global minimization of a non-convex objective is intractable in general, so using the best fixed action in hindsight as a comparator is too ambitious.
Second, as we explain below, even if one uses a proxy for stationarity in lieu of a global minimizer, an informed adversary can still impose $\reg(T) = \Omega(T)$, so the notion of regret minimization must also be re-examined in this setting.

We address both of these problems by extending the \emph{local regret minimization} framework of \cite{hazan_efficient_2017} to the composite problem \eqref{prob:1}.
To do so, we begin by defining the \emph{proximal mapping} of $g$ along the search direction $\bd\in\real^n$ with step-size $\eta > 0$ as
\begin{equation}
\label{eq:prox-res}
T_{\eta}^{g} (\bx; \bd)
	\equiv \mathrm{prox}_{\eta g} \left( \bx - \eta \bd\right)
	= \argmin\nolimits_{\bz \in \real^{n}} \{ \eta g(\bz) + \tfrac{1}{2} \|\bx - \eta\bd - \bz \|^{2} \},
\end{equation}
where $\| \cdot \|$ stands for the Euclidean norm, and the corresponding \emph{prox residual} as
\begin{equation}
\label{eq:1a}
\mathcal{P}_{\eta}^{g}(\bx;\bd)
	= \frac{1}{\eta}\left(\bx - T_{\eta}^{g} (\bx ; \bd) \right).
\end{equation}

\begin{remark}
We note that the purpose behind the use of a general vector $\bd$ in  \cref{eq:prox-res} and  \cref{eq:1a} is  to be able to accommodate for \emph{stochastic} gradients  later on in  \cref{sec:4}. 
\end{remark}
	
As an illustration, let us set  $\bd = \nabla f(\bx)$ and examine \cref{eq:prox-res} and \cref{eq:1a} in the smooth unconstrained and constrained scenarios.
If $g \equiv 0$, then \cref{eq:prox-res} is the gradient descent operator and \cref{eq:1a} reduces to $\mathcal{P}_{\eta}^{g}(\bx;\nabla f(\bx)) = \nabla f(\bx)$.
Likewise, if $g \equiv \delta_{\mathcal{K}}$ for some closed convex subset $\mathcal{K}$ of $\real^{n}$, we get the projected gradient descent in \cref{eq:prox-res} and its corresponding projection residual $\mathcal{P}_{\eta}^{g}(\bx; \nabla f(\bx)) = \eta^{-1} (\bx - \proj_{\mathcal{K}}( \bx - \eta \nabla f(\bx)))$.

%

A fundamental  result in optimization is that $\mathcal{P}_{\eta}^{g}(\bx;\nabla f(\bx)) = 0$ if and only if $\bx$ is a stationary point of \eqref{prob:1}, making the residual quantity $\res_{\eta}^{g}(\bx) \equiv \|\mathcal{P}_{\eta}^{g}(\bx;\nabla f(\bx))\|^{2} \geq 0$  an efficient proxy for the first-order optimality condition (see also \cite[Ch. 10]{B17}).

Motivated by this, it would seem natural to define the regret of an online policy $\bx_{t}$ at time $T$ as the classical measure in non-convex optimization
\begin{equation}
\label{eq:1}
\reg(T)
	\equiv \sum_{t=1}^{T} \res_{\eta}^{g}(\bx_{t})
	= \sum_{t=1}^{T} \left\| \mathcal{P}_{\eta}^{g}(\bx_{t};\nabla f_{t}(\bx_{t}) \right\|^{2}.
\end{equation}
However, as was shown by \cite{hazan_efficient_2017}, it is not difficult for the adversary to impose linear regret by providing a sequence of ``spiked'' non-convex loss functions with large $\|\nabla f_{t}(\bx_{t}) \|$ and small gradient away from each $\bx_{t}$ (for completeness, we provide a simple example in the supplement).
Perhaps more intuitively, one may consider a dynamical system with a time varying function that is only accessible via a stochastic oracle (e.g. GAN as a two-players game), in which case, attaining stationarity through the classical  use of  \cref{eq:1} seems impossible.

Because of this, it is more reasonable to consider a \emph{smoothed}, \emph{local} version of the regret that averages the sequence of loss functions encountered over a sliding window of $w$ consecutive time periods.
Formally, for all $w\in [T]$, consider the sliding average
\begin{equation*}
S_{t,w} (\bx)
	= \frac{1}{w} \sum_{i=t-w+1}^{t} f_{i} (\bx),
\end{equation*}
with the convention $f_{t} \equiv 0$ for $t\leq 0$.
Building on the notion of regret proposed by \cite{hazan_efficient_2017}, the \emph{local regret} of a policy $\bx_{t}$ up to time $T$ with window leghth $w$ is then defined as
\begin{equation}
\reg_w (T) = \sum_{t=1}^{T} \left\| \mathcal{P}^g_{\eta} \left( \bx_{t}; \nabla S_{t,w}(\bx_t)\right)  \right\|^2. \label{eq:8}
\end{equation}
In the above, the sliding window $w$ can be seen as an "effective time unit":
essentially, instead of working with the stream of (potentially volatile) loss functions $f_t$ directly, we work with the average loss over a window of length $w$.
In practice, the sliding window $w$ acts as a "stabilizer" controlling the effects of the noise and variability of the function on the decision making of the optimization protocol; this will become apparent in the sequel.

In the non-composite case, when $g$ is the indicator of a closed convex set, the local regret measure \cref{eq:8} is quantified by the minimax bound of \cite{hazan_efficient_2017} who showed that an informed adversary can impose $\reg_{w}(T) = \Omega(T/w^{2})$.
This bound becomes sublinear in $T$ if $w=\omega(1)$, so this definition provides the required flexibility for a tractable measure of regret.

To further substantiate the motivation for our smoothing approach, we provide four prototypical scenarios in which \cref{eq:8} generalizes standard measures in simpler models:
\begin{itemize}
	\item In the offline case $f_t\equiv f$, we immediately recover the classical measure of  \cref{eq:1}.
	
	\item If $g\equiv0$, we readily obtain $\reg_{w}(T) = (1/w^{2}) \sum_{t=1}^{T}  \|\sum_{i=t-w+1}^{t} \nabla f_{i} (\bx_{t})\|^{2}$, i.e., the original definition of \cite{hazan_efficient_2017} for unconstrained online non-convex problems.
	
	\item If additionally $f_t = F(\cdot,\omega_{t})$ where $F$ is a stochastic objective and $\omega_{t}$ is an i.i.d sequence of random seeds, then $\mathbb{E} \left( \reg_{w}(T)\right) /T \geq \sum_{t=1}^{T} \left\|  \nabla f (\bx_t)\right\| ^{2}$, meaning that local regret minimization leads to stationarity in expectation in unconstrained stochastic models; we will return to this example in \cref{sec:3}.
	\item More generally, as discussed in detail in \cref{sec:41}, if each $f_{t}$ is drawn from an underlying stationary distribution with expectation $f$, and a stopping time $t_{\ast}$ is selected uniformly at random from $[T]$, we will have $\mathbb{E} \big[ \left\| \mathcal{P}^g_{\eta} \left( \bx_{t_*}; \nabla f (\bx_{t_*})\right)  \right\|^2 \big] \leq \mathbb{E}  \left( \reg_w (T)\right) /T$, i.e., local regret minimization implies average stationarity in composite (offline) stochastic problems.
\end{itemize}

We close this section by introducing a measure of variation of the loss functions encountered by the optimizer, and which will be particularly useful in the sequel:

\begin{defin}
[Sliding window variation]
The sliding window variation of a sequence of loss functions $f_{t}$ is 
	\begin{equation}
	\label{eq:V}
	V_w [T] = \sup_{\bx\in \mathrm{dom} g} \left\lbrace \sum_{i=1}^{T} \| \nabla f_{i} (\bx) -  \nabla f_{i-w} (\bx) \|^2\right\rbrace.
	\end{equation}
\end{defin}

An immediate observation is that if the gradients of the functions are bounded (e.g., if $f_t$ is Lipschitz continuous), we automatically have $V_w[T] = O(T)$;
as such, any regret guarantee stated in terms of $V_w[T]$ automatically translates to $O(T)$ in this context.

The main reason that we introduce this variation measure instead of working with a more uniform hypothesis, such as the standard Lipschitz continuity of the objective function, is to account for cases where this quantity is naturally small.
For example, in the routing problem mentioned in \cref{sec:routing} and detailed in the supplemental, $V_w[T]$ corresponds to the variability of the encountered traffic demands at a time-scale of $w$.
As such, if the sliding window $w$ is attuned to the seasonal variability of the process (\eg an hour, a day or a week, depending on granularity), $V_w[T]$ could be considerably smaller than $T$, so the obtained regret bounds would be considerably sharper as a result.

We should also note that, when $w = 1$, $V_w[T]$ boils down to the ``gradual variation'' measure of \cite{CYLMLJZ12} \textendash\ and, indirectly, to the variation budget of \cite{BGZ15}.
The above suggests an interesting interplay between our analysis and regret minimization relative to a dynamic comparator;
this is also part of the reason that we state our results in terms of $V_w [T]$ in the sequel.

\section{The time-smoothed online prox-grad method}
\label{sec:3}
Assuming perfect first-order oracle, we  introduce the \textit{Time-Smoothed Online Prox-Grad Descent} method, cf. \cref{alg:3}, which generalizes the \textit{time-smoothed online gradient descent} method of \cite{hazan_efficient_2017}.

\begin{algorithm}[htbp]
	\caption{Time-smoothed online prox-grad descent}
	\label{alg:3}
	\textbf{Input.} $\bx_1\in\real^n$, $\eta\in (0,1/L) $, $w\in [T]$, $\delta>0$.\\
	\textbf{General step.} For any $t=1,\ldots, T$ do:
	\begin{enumerate}
		\item $f_t:\real^n\rightarrow \real$ is determined;
		\item Set $\bx_{t+1} \leftarrow \bx_t $;
		\item While $ \left\|\mathcal{P}^g_{\eta} \left( \bx_{t+1}; \nabla S_{t,w}(\bx_{t+1})\right) \right\| > \delta/w$ do:
		\begin{enumerate}
			\item Update $ \bx_{t+1} \leftarrow \argmin_{\bz\in\real^n} g(\bz) +\langle \nabla S_{t,w} (\bx_{t+1}), \bz-\bx_{t+1} \rangle + \frac{1}{2\eta} \| \bz - \bx_{t+1}\|^2$;
		\end{enumerate}
	\end{enumerate}
\end{algorithm}

As we show below, \Cref{alg:3} achieves an optimal regret bound of $O\left(\frac{T}{w^2} \right) $ when $V_w[T]$ is bounded by $O(T)$, and executes $O(w^2)$ prox-grad operations.
We note that the bound $O(w^2)$ on the number of prox-grad operations improves the bound $O(Tw)$ established for the simplified case of $g\equiv 0$ by \cite{hazan_efficient_2017}.

\begin{theorem}[Local regret minimization]
	\label{thm:1}
	\cref{alg:3} enjoys the local regret bound 
	\begin{align*}
	\reg_w (T) \leq\frac{2}{w^2}\left( T \delta^2 + V_w[T] \right).
	\end{align*}
\end{theorem}
\begin{theorem}[Oracle queries]
	\label{thm:3}
	Let $\tau_t$ be the number of prox-grad operations at time $t\in [T]$.
	The total number of oracle queries $\tau = \sum_{t=1}^{T} \tau_t $ made by \cref{alg:3} is bounded as
	\begin{align*}
	\tau \leq \frac{2 w^2 (g(\bx_1)  + 2 M)}{\left( 2 - \eta L\right) \eta \delta^2}
	= O(w^{2}).
	\end{align*}
\end{theorem}


We conclude this section by examining the theoretical guarantees of  \cref{alg:3} when $f_t$ is an unbiased stochastic approximation of $f$, so that, implicitly, $\nabla f_t$ is generated via an unbiased SFO.
It should be noted that the SFO must satisfy that $V_w [T]$ is $O(T)$, which effectively bounds the variability of the stochastic gradient; this assumption is  different than the standard variance bound in stochastic  gradient analysis (cf. \cref{def:1}).
\begin{corollary}
	\label{cor:2}
	Suppose that $g\equiv 0$, $\mathbb{E}(\nabla f_t(\bx) - \nabla f(\bx)) = 0$ for any $\bx\in\real^n$, and  that $ V_w [T] \leq c T  $ for some $c>0$.
	Let $\varepsilon>0$, and $t_*\in [T]$ be chosen uniformly from $\{ w,w+1,\ldots,T \}$.
	If $T=2w$ and $w = \left\lceil 2\sqrt{(\delta^2 + c) / \varepsilon}\right\rceil $. 
	Then \cref{alg:3} achieves $\mathbb{E} \left( \left\|\nabla f (\bx_{t_*}) \right\|^2 \right) \leq \varepsilon$ with at most $O(\varepsilon^{-1}) $ prox-grad operations and  $O( \varepsilon^{-3/2}) $ SFO calls. 
\end{corollary}
Note that the complexities reported in \cref{cor:2} match those obtained for the state-of-the-art \textit{Prox-SpiderBoost} method proposed by \cite{WJZLT19}, but under a different procedure using more stringent assumptions (boundedness of $f$ and that $V_w [T]$ is $O(T)$).
We stress that the Prox-SpiderBoost method is only applicable to stochastic problems, and as such, it has no online guarantees, unlike \cref{alg:3}.

The proofs of \cref{thm:1,thm:3}, and of  \cref{cor:2}, are deferred to the supplemental.

\section{Stochastic time-smoothed online prox-grad method}
\label{sec:4}

\subsection{Method and Analysis}
Moving forward from the deterministic guarantees of \cref{alg:3}, we proceed to consider a more flexible framework that only posits access to a \acdef{SFO}. 
Specifically, following \cite{NJLS09}, we assume that it is possible to generate an \acs{iid} sequence of random seeds $\omega_1,\omega_2,\ldots,$ that are concurrently used as input to an \ac{SFO} as follows:
\begin{defin} [Stochastic first-order oracle]
	\label{def:1}
	A \acdef{SFO} is a function $\mathcal{S}_{\sigma}$ such that, given a point $\bx\in\R^{n}$, a random seed $\omega$, and a smooth function $h\colon\real^{n}\to\real$ satisfies:
	\begin{enumerate}
		\item $\mathcal{S}_{\sigma}(\bx;\omega,h)$ is unbiased relative to $ \nabla h (\bx)$: $\mathbb{E} \left( \mathcal{S}_{\sigma}(\bx;\omega,h) - \nabla h (\bx) \right) = 0 $;
		\item $\mathcal{S}_{\sigma}(\bx;\omega,h)$ has variance bounded by $\sigma>0$: $\mathbb{E} \left(\| \mathcal{S}_{\sigma}(\bx;\omega,h) - \nabla h (\bx) \|^2 \right) \leq \sigma^2$.
	\end{enumerate}
\end{defin}
With all this hand, the heuristics of the proposed stochastic prox-grad method are as follows:
(i) $f_t$ is determined;
(ii) successive \ac{SFO} queries generate a noisy descent process in an inner loop until a $\delta/w$-stationary point is reached.
In detail, the algorithm is presented in pseudocode form below:

\begin{algorithm}[htbp]
	\caption{Time-smoothed online stochastic prox-grad method}
	\label{alg:4}
	\textbf{Input.} $\bx_1\in\real^n$, $\eta\in (0,1/L) $, $w\in [T]$, $\delta>0$.\\
	\textbf{Initialization.} $\tilde{\nabla} S_{i,w} (\bx_1) = \bo $ for all $i\leq 0$.\\
	\textbf{General step.} For any $t=1,2, \ldots, T$ do:
	\begin{enumerate}
		\item Function is updated to $f_t:\real^n\rightarrow \real$;
		\item Sample $\tilde{\nabla} f_t (\bx_t) \leftarrow \mathcal{S}_{\sigma/w} (\bx_t;\omega, f_t)$;
		\item Set $\tilde{\nabla} S_{t,w} (\bx_t) = \tilde{\nabla}S_{t-1,w} (\bx_t) + \frac{1}{w} (\tilde{\nabla} f_t (\bx_t) - \tilde{\nabla} f_{t-w} (\bx_t) )$;
		\item Set $\by^1_t = \bx_t $, $G_t^1 = \tilde{\nabla} S_{t,w} (\bx_t)$, $k=1$;
		\item While $ \left\| \mathcal{P}^{g}_{ \eta } \left( \by^k_t; G_t^k \right) \right\| > \delta/w$ do:
		\begin{enumerate}
			\item Update $ \by^{k+1}_t = \argmin_{\bz\in\real^n} g(\bz) +\langle G_t^k, \bz-\by^k_t \rangle + \frac{1}{2\eta} \| \bz - \by^k_t\|^2$;
			\item Sample $\tilde{\nabla} f_i (\by^{k+1}_t) \leftarrow \mathcal{S}_{\sigma/w} (\by^{k+1}_t; \omega, f_i)$ for any $i=t-w+1,\ldots,t$;
			\item Set $G_t^{k+1} = \frac{1}{w}\sum_{i=t-w+1}^t \tilde{\nabla} f_i (\by^{k+1}_t) $;
			\item Set $k \leftarrow k+1$;
		\end{enumerate}
		\item Set $\bx_{t+1} = \by^k_t $ and $\tilde{\nabla} S_{t} (\bx_{t+1}) =G_t^{k}$.
	\end{enumerate}
\end{algorithm}
The process of \cref{alg:4} might be better understood by comparing  it to offline stochastic variance reduction methods  (SVR); see e.g., \cite{FLLZ18,MT19,WJZLT19,YSC19}, and references therein.
For these methods, which usually implement a non-diminishing step-size policy in the non-convex setting, a batch-size variance relation is required in order to achieve the methods' guarantees.

 \cref{alg:4} takes a different approach in this context by, instead of stating  this connection in the analysis, it explicitly links the batch-size ($w$ mimics the role of the batch-size)  to  the variance of the SFO in the scheme itself.
The affinity of  \cref{alg:4} to   SVR methods is further expressed when  considering its guarantees in the offline scenario of $f_t\equiv f$.
Then, \cref{alg:4} achieves the best known SFO complexity as that obtained by SVR methods; see  our \cref{sec:41} for additional details.\\

Before stating \cref{alg:4}'s guarantees, let us  first define the algorithm's natural filtration: 
For all $t\geq 1$, the filtration $\mathcal{F}_{t}$ includes all gradient feedback up to, but not including, the execution of step 2 at stage $t$.
In particular, it includes $f_t$, $\bx_t$ and $\tilde{\nabla}S_{t-1} (\bx_t)$, but it does not include $\tilde{\nabla} f_t (\bx_t)$.
%


With all this in hand, we now state our main results.
Denote by $\tau_{t}$ the number of times the condition in step 5 at $t$-th iteration is checked, that is the number of prox-grad operations at the $t$-th iteration, and let $\tau = \sum_{t\in [T]} \tau_t$.
We begin by establishing that \cref{alg:4} almost surely executes a finite number of prox-grad operations provided that $\delta$ is not too small.
\begin{theorem}[Oracle queries]
	\label{thm:4}
	Let $t\in [T]$ and let the filtration $\mathcal{F}_{t}$ be given.
	Suppose that the inputs $\delta$ and $\eta$ satisfy that
	\begin{equation}
	\label{eq:23}
	\delta^2 > \dfrac{2\sigma^2}{\eta \left( 1- \eta L\right) }.
	\end{equation}
	Then $\tau_t$ and $\tau$ are almost surely finite, and 
	\begin{align*}
	\mathbb{P} (\tau_t > K) &\leq\frac{(h^1_t + M)w^2}{2\left( \eta\left( 1 - \eta L\right)\delta^2 - 2 \sigma^2 \right) K} = O(1/K), \qquad \forall K\geq 1.
	\end{align*}
\end{theorem}

Next we provide a tight bound on the expected local regret in terms of $V_w [T]$; recall that under the standard assumptions of bounded feasible domain or Lipschitz continuity of $f_t$, $V_w [T]$ is bounded by $O(T)$, in which case we have that $\mathbb{E} \left[ \reg_w (T) \right] $ achieves the optimal local regret bound of $O\left(\frac{T}{w^2} \right)  $.
\begin{theorem}[Local regret minimization]
	\label{thm:2}
	\cref{alg:4} enjoys the average local regret bound 
	\begin{align*}
	\mathbb{E} \left[ \reg_w (T) \right] \leq 2\left( \frac{T}{w^2} \right) \left( \delta^2 + 7\sigma^2\right) + \frac{6}{w^2} V_w [T].
	\end{align*}
\end{theorem}

The local regret bound established in \cref{thm:2}, and the almost sure termination in finite time proved in \cref{thm:4}, leave the question of the number of prox operation still unattended.
To answer this nontrivial question, we require more control of the random processes originating from the SFO in the form of the following assumption on the noise.
\begin{assumption}
	\label{ass:3}
	Given any point $(\bx,\omega)\in\real^n\times \Omega$ and a function $h:\real^n\rightarrow\real$, the stochastic first-order oracle $\mathcal{S}_\sigma$ satisfies that $\| \mathcal{S}_{\sigma}(\bx;\omega,h) - \nabla h (\bx) \| \leq \sigma;$
\end{assumption}
\Cref{ass:3} is not uncommon in the stochastic setting, even in convex problems, see e.g., \cite{KLBC19,LO19,JNN19}, and references therein.
We emphasize that \cref{thm:2,thm:4} \textit{do not} require, \textit{nor} assume, that \cref{ass:3} holds true.

The next theorem states that \cref{alg:4} executes $O( w^2)$ prox operations and  $O(w^3)$ SFO calls. 
\begin{theorem}[Iteration bound]
	\label{thm:5}
	Suppose that \cref{ass:3} holds true, and that $\eta\in (0,1/(L+1))$, $\delta^2 >  \sigma^2/\eta(1- \eta (L+1))$.
	Then  the number of SFO calls is $O(w\tau)$ with
	\begin{equation}
	\label{eq:21}
	\tau = \sum_{t=1}^{T} \tau_t \leq \frac{2 w^2 (g(\bx_1) + 2 M )}{(1- \eta (L+1))\eta\delta^2 - \sigma^2 } = O(w^2).
	\end{equation}
\end{theorem}

\begin{remark}
Under the conditions of \cref{thm:5},  both \cref{thm:4} and \cref{thm:5} hold true.
\end{remark}

\subsection{Implications to Offline Stochastic Optimization}
\label{sec:41}
This section considers the reduction of our model to an offline stochastic non-convex composite optimization problem by examining our results when $f_t \equiv f$ for any $t\in [T]$. 
In this scenario, where the goal is to obtain an $\varepsilon$-stationary point $\bx_*\in\real^n$  satisfying that $\left\| \mathcal{P}(\bx_*; \nabla f (\bx_*))\right\|^2 \leq \varepsilon$ (cf. \cite[Ch. 2]{B17}), our sliding average $S_{t,w} (\bx)$ is reduced to the objective function itself, and the local regret measure $\reg_w (T)$ is  reduced to the standard  sum of prox-residuals  in the consecutive points generated by the algorithm.
\cref{alg:4} itself takes the form of a stochastic prox-grad type method in which $w$  calls to the SFO are used to approximate the  gradient at each iteration.
This resulting scheme bare some resembles to  variance reduction techniques appearing in \cite{MT19,WJZLT19,YSC19}, where here, $w$ seemingly takes the role of the batch-size, and the process of \cref{alg:4} enforces the relation between the SFO's variance and $w$.

The connection between  \cref{alg:4} and SVR methods is further supported by the  $O(M\sigma \varepsilon^{-3/2})$ SFO calls complexity guarantee for obtaining a $\varepsilon$-stationary point in expectation, which we will derive shortly. 
This complexity is currently the best known (sometimes written as $O(M\sigma \varepsilon^{-3})$ due to square-difference in the stationarity definition), and can only be obtained by SVR methods; see  the already mentioned \cite{ACDFSW19} for details.


Although obtained as a by-product, our offline-related result are of an independent interest and contribution, as, apart from providing a new connection between online learning and offline stochastic optimization, we also derive a new stochastic method with the best known guaranteess under  different model assumptions and procedure compared to the SVR literature.

It should be noted though that our assumptions, albeit standard in online optimization, are more restrictive compared to the related stochastic (offline) optimization literature (e.g., \cite{WJZLT19}), as the former facilitate guarantees, first and foremost, for our \textit{online} stochastic model. 
Indeed, methods for stochastic problems cannot address the adversarial online settings we study here. 
Notwithstanding, our complexity results suggest new scheme's design directions to explore in the development of (offline) stochastic methods, encouraging future study on the matter, that is unfortunately out of the scope if this paper.

Let us now derive the aforementioned guarantees, proofs are provided in the supplemental.

\begin{theorem}
	\label{thm:6}
	Let $\varepsilon>0$, and $t_*$ be chosen uniformly from $\{ w,w+1,\ldots,T \}$.
	Suppose that $ V_w [T] \leq c T /6 $ for some $c>0$.
	Then 
	$\mathbb{E} \left( \left\| \mathcal{P}(\bx_{t_*}; \nabla f (\bx_{t_*})) \right\|^2 \right) \leq \frac{2T \left( \delta^2 + 7\sigma^2 + c \right)}{(T-w) w^2}. $
\end{theorem}
From \cref{thm:5} and \cref{thm:6} we obtain the desired guarantees. 
\begin{corollary}
	\label{cor:1}
	Let $\varepsilon>0$, and $t_*\in [T]$ be chosen uniformly from $\{ w,w+1,\ldots,T \}$.
	Suppose that $ V_w [T] \leq c T /6 $ for some $c>0$.
	If $T=2w$ and $w = \left\lceil 2\sqrt{(\delta^2 + 7\sigma^2 + c) / \varepsilon}\right\rceil $. 
	Then \cref{alg:4} achieves $\mathbb{E} \left( \left\| \mathcal{P}(\bx_{t_*}; \nabla f (\bx_{t_*})) \right\|^2 \right) \leq \varepsilon$.
	Additionally, under the conditions of \cref{thm:5} with $\delta^2 = 2 \eta \sigma^2/(1-\eta(L+1))$, \cref{alg:4} executes at most $O(M \sigma\varepsilon^{-3/2}) $ SFO calls. 
\end{corollary}

\section{Conclusions and future work}

Our aim in this paper was to develop an online prox-grad methodology for stochastic non-convex online optimization problems with constraints and regularization (possibly non-smooth).
In this regard, the proposed framework achieves the min-max optimal bounds for local regret minimization while at the same time bounding the number of overall operator queries. 
From a top-down perspective, this departure from standard notions of regret suggests various extensions based on different notions of local regret, ranging from measures of stationarity in offline non-convex analysis, to proxies for constraint qualification in problems with sufficient regularity.
Additionally, our reductions to the offline stochastic setting suggest new and interesting schemes to address stochastic non-convex optimization problems.
We defer these questions to  future research.

\appendix

\section{Motivating examples}
\label{app:TAP}

\subsection{A conceptual approach for non-convex games}
We  extend here  the solution concept for non-convex $m$-player games with \textit{smoothed local equilibrium}  proposed by \cite{hazan_efficient_2017}  to be valid in our \textit{stochastic} \textit{composite} game setup.
We emphasize that the guarantees we present in this section are also valid for when each player only has access to a stochastic first-order oracle, making it closer to practical use.

To model the multi-player setting,  consider $m$  problems of the form \eqref{prob:1} corresponding to each of the players, where every player $i$ observes her online part of her objective function 
\begin{equation}
\label{eq:cost_game}
f_t^i (\bz) := f(\bx_t^1,\dots, \bx_t^{i-1}, \bz,\dots,\bx_t^m),
\end{equation}
and then decides on $\bx_{t+1}^i$. 

It is sometimes desirable to induce specific properties in the game, this is fully supported by our model \eqref{prob:1}.
For example: (i) to  incur risk-aversion, the regularizer of each player $g^i$ can be chosen accordingly, e.g., $L^1$-norm; (ii) to ensure a meaningful solution, such as the \textit{global minimax point} condition defined by \cite{JNJ19}, restriction of the decision set to a compact convex set can be applied.

In our non-convex setting, obtaining the global measure of Nash equilibrium is beyond reach, and may not exist at all  \cite[Prop. 6]{JNJ19}.
Thus, a different, local, measure for equilibrium is essential.
This topic is already receiving much attention in the literature, for example,  for a multi-player non-convex games, \cite{PG11} proposes the local \textit{quasi-Nash equilibrium} measure defined using KKT conditions.
In the case of a (two-players) minmax game (e.g., GANs) for example, local measure is defined as the stationarity (first-order condition) of both players in  the very recent \cite{NSHLR19,JNJ19}.
For additional details, we refer to the works alluded above.

We follow the smoothed local equilibrium approach \cite[Sec. 6]{hazan_efficient_2017} , and  extend it here to our composite model.
This approach comes naturally from assuming that the players take into account the behavior history of the other players.
Other than that, it allows for a tractable notion of equilibrium.

The \textit{smoothed local equilibrium} is defined for the joint cost function  \eqref{eq:cost_game} as follows, where $ S^i_{t,w}(\bx) = \frac{1}{w} \sum_{j=t-w+1}^{t} f^i_{j} (\bx)$.
\begin{defin}[smoothed local equilibrium]
	Let $\eta>0,w\geq 1$. 
	For an $m$-player iterative game with cost functions as in \eqref{eq:cost_game}, 
	a joint strategy at iteration $t>0$, $(\bx_t^1,\dots, \bx_t^{i-1}, \bx_t^i,\dots,\bx_t^m)$, is an $\varepsilon$-$(\eta,w)$ \textit{smoothed local equilibrium} with respect to the history of $w$-iterates if:
	\begin{equation}
	\label{eq:equib}
	\left\| \mathcal{P}_{\eta}^{g^i}(\bx^i_{t};\nabla S^i_{t,w}(\bx^i_{t}) \right\|^{2}\leq \varepsilon \qquad \forall i\in [m].
	\end{equation}
\end{defin}
Denote by $\reg^i_w (T) $ the local regret (cf. \cref{eq:8}) of the $i$-th player.
We first derive a guarantee for when each player has access to a perfect first-order oracle (using \cref{thm:1}).
\begin{theorem}[Equilibrium with perfect oracle]
		Let the sequence $(\bx_t^1,\dots, \bx_t^{i-1}, \bx_t^i,\dots,\bx_t^m)$, $t=1,\dotsc,T$ be generated by running \cref{alg:3} for all players simultaneously with input $\eta>0$  and $w = \lceil 2 k  (\delta^2 +c) \varepsilon^{-1/2} \rceil$, given that the online function is determined by \eqref{eq:cost_game}.
		Suppose that $V_w[T] \leq cT $ for some $c>0$.
		Then  there exists $t^*\geq w$ such that \eqref{eq:equib} holds true.
\end{theorem}
\begin{proof}
There exists a $t^*\geq w$ such that
\begin{align*}
\sum_{i=1}^{k} \left\| \mathcal{P}_{\eta}^{g^i}(\bx^i_{t^*};\nabla f^i_{t^*}(\bx^i_{t^*}) \right\|^{2} &\leq \frac{1}{T-w} \sum_{i=1}^{k}\sum_{t=w}^{T}\left\| \mathcal{P}_{\eta}^{g^i}(\bx^i_{t};\nabla f^i_{t}(\bx^i_{t}) \right\|^{2} \\
&\leq \frac{1}{T-w} \sum_{i=1}^{k} \reg^i_w (T).
\end{align*}
Thus, if each player has access to a perfect first-order oracle and $V_w[T] \leq cT$, then by \cref{thm:1}
\begin{equation*}
\sum_{i=1}^{k} \left\| \mathcal{P}_{\eta}^{g^i}(\bx^i_{t^*};\nabla f^i_{t^*}(\bx^i_{t^*}) \right\|^{2} \leq \frac{1}{T-w} \sum_{i=1}^{k} \frac{2}{w^2}\left( T \delta^2 + V_w[T] \right) \leq \frac{2 k T (\delta^2 +c)}{(T-w) w^2}.
\end{equation*}
Consequently, by setting $T = w^2$ and $w = \lceil 2 k  (\delta^2 +c) \varepsilon^{-1/2} \rceil$ we obtain
\begin{equation*}
\sum_{i=1}^{k} \left\| \mathcal{P}_{\eta}^{g^i}(\bx^i_{t^*};\nabla f^i_{t^*}(\bx^i_{t^*}) \right\|^{2}  \leq \frac{2 k  (\delta^2 +c)}{(w-1) w} \leq \varepsilon,
\end{equation*}
as desired.
\end{proof}
By similar arguments, we derive the guarantees for when players have access via a stochastic first-order oracle, only now we utilize \cref{thm:2}; we implicitly assume here that all the conditions of \cref{thm:2} are satisfied.
\begin{theorem}[Equilibrium with stochastic first-order oracle]
	Suppose that the sequence $(\bx_t^1,\dots, \bx_t^{i-1}, \bx_t^i,\dots,\bx_t^m)$, $t=1,\dotsc,T$ is generated by running \cref{alg:3} for all players simultaneously with input $\eta>0$  and $w = \lceil \frac{2k  \left( \delta^2 + 7\sigma^2 + 6c\right)}{\sqrt{\varepsilon } } \rceil $, given that the online function is determined by \eqref{eq:cost_game}.
	Suppose that $V_w[T] \leq cT $ for some $c>0$.
	Then  there exists $t^*\geq w$ such that \eqref{eq:equib} holds true in expectation.
\end{theorem}
\begin{proof}
	There exists a $t^*\geq w$ such that
	\begin{align*}
	\sum_{i=1}^{k} \left\| \mathcal{P}_{\eta}^{g^i}(\bx^i_{t^*};\nabla f^i_{t^*}(\bx^i_{t^*}) \right\|^{2} &\leq \frac{1}{T-w} \sum_{i=1}^{k}\sum_{t=w}^{T}\left\| \mathcal{P}_{\eta}^{g^i}(\bx^i_{t};\nabla f^i_{t}(\bx^i_{t}) \right\|^{2} \\
	&\leq \frac{1}{T-w} \sum_{i=1}^{k} \reg^i_w (T).
	\end{align*}
	Thus, by taking expectation and using the fact that $V_w[T] \leq cT$, we obtain from \cref{thm:2} that
	\begin{align*}
	\sum_{i=1}^{k} \mathbb{E} \left\| \mathcal{P}_{\eta}^{g^i}(\bx^i_{t^*};\nabla f^i_{t^*}(\bx^i_{t^*}) \right\|^{2}
	&\leq \frac{1}{T-w} \sum_{i=1}^{k} 2\left( \left(\frac{T}{w^2} \right)  \left( \delta^2 + 7\sigma^2\right) + \frac{6}{w^2} V_w [T]\right)
	\\
	&=  \frac{2k T \left( \delta^2 + 7\sigma^2 + 6c\right)}{(T-w) w^2}.
	\end{align*}
	Consequently, by setting $T = w^2$ and $w = \lceil \frac{2k  \left( \delta^2 + 7\sigma^2 + 6c\right)}{\sqrt{\varepsilon } } \rceil$ we obtain
	\begin{equation*}
	\sum_{i=1}^{k}  \mathbb{E} \left\| \mathcal{P}_{\eta}^{g^i}(\bx^i_{t^*};\nabla f^i_{t^*}(\bx^i_{t^*}) \right\|^{2}  \leq \frac{2k  \left( \delta^2 + 7\sigma^2 + 6c\right)}{(w-1) w } \leq \varepsilon,
	\end{equation*}
	as desired.
\end{proof}

\subsection{The \acl{ONTAP}}
\label{app:TAP1}
Referring to \cite{BG92} and \cite{SS08} for an introduction to the topic, the key objective in \aclp{TAP} is the optimal allocation of traffic over a given network with variable traffic inflows.
To state this precisely, consider a directed multi-graph $\mathcal{G} = (\mathcal{V},\mathcal{E})$ with vertex set $\mathcal{V}$ and edge set $\mathcal{E}$.
Embedded in this network is a set of origin-destination (O/D) pairs $(o_{i},d_{i}) \in \mathcal{V}\times\mathcal{V}$, $i \in \mathcal{N} = \{1,2,\dotsc,N\}$, each routing a (possibly random) quantity of traffic from $o_{i}$ to $d_{i}$ via a set of paths $\mathcal{P}_{i}$ in $\mathcal{G}$.
Writing $\mathcal{K}_{i} = \Delta(\mathcal{P}_{i})$ for the simplex spanned by $\mathcal{P}_{i}$, a \emph{trafic allocation vector} for the $i$-th O/D pair is defined to be a vector $\bx_{i} = (x_{i,p_{i}})_{p_{i}\in\mathcal{P}_{i}} \in \mathcal{K}_{i}$ with each $x_{i,p_{i}}$ denoting the fraction of the traffic of the $i$-th O/D pair that is routed via $p_{i}$.
Then, collectively, a \emph{traffic allocation profile} is an ensemble $\bx = (\bx_{1},\dotsc,\bx_{N})$ of such vectors belonging to the product space $\mathcal{K} = \prod_{i} \mathcal{K}_{i}$.

In this general context, the cost (delay, latency, etc.) of routing a certain amount of traffic via a given path $p_{i}$ is a function $\ell_{p_{i}}(\bx;\blambda)$ of the chosen allocation profile $\bx\in\mathcal{K}$ and the set of \emph{traffic demands} $\blambda = (\lambda_{1},\dotsc,\lambda_{N})$ of each O/D pair.%
These demands are typically assumed to follow a non-stationary probability distribution (e.g., accounting for diurnal variations in an urban traffic network), leading to the \acdef{ONTAP} stated below:
\begin{equation}
\label{eq:ONTAP1}
\tag{OnTAP}
\begin{aligned}
\textrm{minimize}
&\quad
\ell_{t}(\bx)
= \sum_{i\in\mathcal{N}} \sum_{p_{i}\in\mathcal{P}_{i}} x_{i,p_{i}} \ell_{p_{i}}(\bx;\blambda_{t})
+ \mu \|\bx\|_{1}
\\
\textrm{subject to}
&\quad
\bx \in \mathcal{K}. 
\end{aligned}
\end{equation}
In the above formulation, the sparsity-inducing $L^{1}$ term is intended to ``robustify'' solutions by minimizing the overall number of paths employed.
The cost functions $\ell_{p_{i}}$ are sums of positive polynomials (described below), so they are smooth over $\mathcal{K}$ but may otherwise be non-convex.
As such, \eqref{eq:ONTAP1} can be cast in the framework of \eqref{prob:1} by taking $g = \delta_{\mathcal{K}} + \mu\|\cdot\|_{1}$ with $\delta_{\mathcal{K}}$ denoting the convex indicator of $\mathcal{K}$.

Let us now detail the definition of the cost functions $\ell_{p_{i}}$ for \eqref{eq:ONTAP1}.
For simplicity, we will suppress the O/D index $i\in\mathcal{N}$, i.e., we will treat the problem as a single-O/D one;
this doesn't play a major role in the sequel and only serves to make the notation ligther.

To begin, given a traffic allocation vector $\bx \in \mathcal{K}$ and an inflow rate $\lambda$, the \emph{traffic load} carried by edge $e\in\mathcal{E}$ is defined to be the total traffic routed via the edge in question, i.e.,
\begin{equation}
y_{e}
	\equiv y_{e}(\bx;\lambda)
	= \lambda \sum_{p:p \ni e} x_{p},
\end{equation}
and we write $\by = (y_{e})_{e\in\mathcal{E}}$ for the corresponding \emph{load profile} on the network.
Given all this, the cost (delay, latency, etc.) experienced by an infinitesimal traffic element traversing edge $e$ is given by a non-decreasing continuous \emph{cost function} $\ell_{e}\colon \mathbb{R}_{+} \to \mathbb{R}_{+}$;
more precisely, if $\by \equiv \by(\bx;\lambda)$ is the load profile induced by a traffic allocation profile $\bx\in\mathcal{K}$ and a traffic demand $\lambda$, the incurred cost on edge $e\in\mathcal{E}$ is simply $\ell_{e}(y_{e})$.
Hence, the associated cost for path $p\in\mathcal{P}$ will be
\begin{equation}
\ell_{p}(\bx;\lambda)
	\equiv \sum_{e\in p} \ell_{e}(y_{e}(\bx;\lambda))
	= \sum_{e\in p} \ell_{e}\left( \lambda \sum_{p':p'\ni e} x_{p'} \right).
\end{equation}
In urban traffic networks, the cost functions $\ell_{e}$ are typically non-decreasing positive polynomials fitted to appropriate statistical data;
a common choice is the so-called ``quartic BPR'' model $\ell_{e}(y_{e}) = a_{e} + b_{e} y_{e}^{4}$ of the US Bureau of Public Roads (BPR), but this is beyond our scope.


\section{Regretfulness when $w=1$}
\label{sec:aexm}
For completeness, we provide a simple example for when the "standard" stationarity measure \cref{eq:1}, obtained from the local regret when $w=1$, fails.
The bound $O(T/w^{2})$ established in \cite[Thm. 2.7]{hazan_efficient_2017} is proved via a similar example.

Suppose that $g(x) = \delta_{[-1,1]} (x)$ is the indicator function for the set $[-1,1]$, and that
\begin{equation*}
f_t (x) = \begin{cases}
-x & \text{ with probability } 0.5, \\
x & \text{ with probability } 0.5.
\end{cases}
\end{equation*}
Then
\begin{equation*}
\mathbb{E} \reg_1 (T) = \mathbb{E}  \sum_{t=1}^{T} \left\| \mathcal{P}_{\eta}^{g}(\bx_{t};\nabla f_{t}(\bx_{t}) \right\|^{2} \geq O(T).
\end{equation*}

\section{Fundamental Properties}
\label{sec:a2}
Throughout the analysis, we  utilize  fundamental properties of the prox operator for $L$-smooth functions. 
The  descent lemma (see e.g., \cite[Lem. 5.7]{B17}) and the sufficient decrease property of the prox-grad operator  (cf. \cite[Lem. 10.4]{B17}) are given as follows.
\begin{lem}[Descent lemma]
	\label{lem:des}
	Let $f:\real^n\rightarrow(-\infty,\infty]$ be an $L$-smooth function ($L\geq 0$) over a convex set $C\subseteq\real^n$.
	Then for any $\bx,\by\in C$,  $ 	f(\by) \leq f(\bx) + \langle \nabla f(\bx), \by - \bx\rangle + \frac{L}{2} \| \bx - \by \|^2.$
\end{lem}
\begin{lem}[Sufficient decrease property]
	\label{lem:suf}
	Let $h:\real^n \rightarrow\real \cup\{\infty\}$ be a proper, convex,  l.s.c function, and $f:\real^n\rightarrow(-\infty,\infty)$ be an $L$-smooth function ($L\geq 0$) over $\dom h$.
	Then for any $\bx\in\int\dom h$ and $\eta\in (0, L/2)$ it holds for $\bx^+ = \mathrm{prox}_{\eta h} (\bx - \eta \nabla f(\bx))$ that
	\begin{equation*}
		h(\bx) + f(\bx) - h(\bx^+) - f(\bx^+) \geq \eta \left( 1 - \dfrac{\eta L}{2}\right)  \left\| \dfrac{1}{\eta} \left( \bx^+ - \bx\right) \right\|^2.
	\end{equation*}
\end{lem}
We also use a trivial, yet essential, property of the prox-grad mapping. 
\begin{lem}
	\label{lem:2}
	For any $\bx,\bd_1,\bd_2\in\real^n$ and $\eta>0$  it holds that
	\begin{equation*}
		\left\| \mathcal{P}_{\eta}^{g} (\bx ; \bd_1 + \bd_2) \right\| \leq \left\| \mathcal{P}_{\eta}^{g} (\bx ; \bd_1) \right\|  + \left\| \bd_2 \right\|.
	\end{equation*}
\end{lem}
\begin{proof}
	By the triangle inequality and non-expensiveness of the prox operator (cf. \cite[Theorem 6.42]{B17})
	\begin{align*}
	\left\| \mathcal{P}^{g}_{ \eta } (\bx;\bd_1 + \bd_2) \right\| -  \left\| \mathcal{P}^{g}_{ \eta } (\bx; \bd_1) \right\| &\leq  \left\| \mathcal{P}^{g}_{ \eta } (\bx;\bd_1+\bd_2) - \mathcal{P}^{g}_{ \eta } (\bx; \bd_1) \right\| \\
	&\leq \frac{1}{\eta}\left\| \left( \bx - \eta (\bd_1 + \bd_2)\right) - \left( \bx - \eta \bd_1 \right)  \right\| = \| \bd_2 \|.
	\end{align*}
\end{proof}

\section{Proofs of \cref{sec:3}}
\label{sec:3a}
\begin{proof}[Proof of  \cref{thm:1}]
	Note that
	\begin{align*}
	S_{t} (\bx) = \frac{1}{w} \sum_{i=t-w+1}^{t} f_{i} (\bx) 
	&= S_{t-1} (\bx) + \frac{1}{w}(f_t(\bx) - f_{t-w}(\bx)) .
	\end{align*}
	Setting $h_1  = S_{t-1}$, $h_2 = \frac{1}{w}(f_t - f_{t-w})$, applying Lemma \ref{lem:2} and the triangle inequality yields
	\begin{align*}
	\left\|  \mathcal{P} (\bx_t ; \nabla S_{t} (\bx_t)) \right\|
		&= \left\| \mathcal{P}(\bx_t; \nabla (h_1 + h_2 )(\bx_t)) \right\|
		\\
		&\leq \left\| \mathcal{P}(\bx_t ; \nabla S_{t-1} (\bx_t)) \right\|  +\frac{1}{w} \left\| \nabla f_t(\bx_t) - \nabla f_{t-w}(\bx_t) \right\|.
	\end{align*}
	By the definition of the method, i.e. $ \left\| \mathcal{P}(\bx_t ; \nabla S_{t-1} (\bx_t)) \right\| \leq \frac{\delta}{w}$, we thus have that
	\begin{align*}
	\left\|  \mathcal{P} (\bx_{t} ; \nabla S_{t} (\bx_t)) \right\|  \leq \frac{\delta}{w}  + \frac{1}{w} \left\| \nabla f_t(\bx_t) - \nabla f_{t-w}(\bx_t) \right\|, \qquad \forall t\in [T],
	\end{align*}
	and consequently, for any $t\in [T]$,
	\begin{align*}
	\left\|  \mathcal{P}(\bx_{t} ; \nabla S_{t} (\bx_t)) \right\|^2  
	&\leq \frac{2\delta^2}{w^2}  + \frac{2}{w^2} \left\| \nabla f_t(\bx_t) - \nabla f_{t-w}(\bx_t) \right\|^2.
	\end{align*}
	Summing over $t = 1,\ldots,T$, then results with
	\begin{align*}
	\reg_w (T) = \sum_{t=1}^{T} \left\|  \mathcal{P}(\bx_{t} ; \nabla S_{t} (\bx_t)) \right\|^2  \leq \frac{2}{w^2}\left( T \delta^2  +  V_w [T] \right).
	\end{align*}
\end{proof}
To prove that Algorithm \ref{alg:3} executes $O(w^2)$ prox-grad calls, we require a sufficient decrease property that is given next.
\begin{lem}[Sufficient decrease property]
	\label{lem:3}
	Let $t\in [T]$, and let $\tau_{t}$ be the number of times step 3 is executed at the $t$-th iteration.
	Then 
	\begin{equation*}
	S_{t,w} (\bx_t) + g(\bx_t)  - S_{t,w} (\bx_{t+1})  - g(\bx_{t+1})  \geq \tau_t \left( \eta - \frac{\eta^2 L}{2}\right) \frac{\delta^2}{w^2}, \qquad \forall t\in [T].
	\end{equation*}
\end{lem}
\begin{proof}
	Denote the  sequence generated in the inner loop at time $t\in [T]$ by
	$$\by_t^0 = \bx_t, \quad \by_t^{k+1} = \argmin_{\bz\in\real^n} g(\bz) +\langle \nabla  S_{t} (\by_t^k), \bz-\by_t^k \rangle   + \frac{1}{2\eta} \| \bz - \by_t^k\|^2, \qquad k=0,1,\ldots,\tau_t-1,$$
	and note that $\by_t^{\tau_t} = \bx_{t+1}$.
	By the sufficient decrease property of the prox-grad operator (cf.  Lemma \ref{lem:suf}), and the stopping criteria of the inner loop, we have that for all $k=0,1,\ldots,\tau_t-1$
	\begin{equation}
	\label{eq:7a}
	S_{t} (\by_t^k) + g(\by_t^k)  - S_{t} (\by_t^{k+1})  - g(\by_t^{k+1})  \geq \left( \eta - \frac{\eta^2 L}{2}\right) \left\|  \mathcal{P} (\by_t^k ; \nabla S_t (\by_t^k)) \right\|^2 \geq \left( \eta - \frac{\eta^2 L}{2}\right) \frac{\delta^2}{w^2}.
	\end{equation}
	Summing \eqref{eq:7a} over $k =  0,1,\ldots,\tau_t-1$, then yields
	\begin{align*}
	S_{t} (\bx_t) + g(\bx_t)  - S_{t} (\bx_{t+1})  - g(\bx_{t+1})
		&= S_{t} (\by_t^0) + g(\by_t^0)  - S_{t} (\by_t^{\tau_t})  - g(\by_t^{\tau_t})
		\\
		&\geq \tau_t \left( \eta - \frac{\eta^2 L}{2}\right) \frac{\delta^2}{w^2}
	\end{align*}
which completes our proof.
\end{proof}

We will now bound the number of prox-grad iterations executed by Algorithm \ref{alg:3}.
\begin{proof} [Proof of \cref{thm:3}]
	Recall that $S_{0} (\bx_0) \equiv 0$, and $S_{t} (\bx) = \frac{1}{w}(f_t(\bx) - f_{t-w}(\bx)) +S_{t-1}  (\bx) $.
	Thus,
	\begin{align*}
	S_{T} (\bx_T) &= \sum_{t=1}^{T} \left( S_{t} (\bx_{t}) -S_{t-1}  (\bx_{t-1}) \right) \\
	&=  \frac{1}{w}\sum_{t=1}^{T} \left( f_t(\bx_t) - f_{t-w}(\bx_t) \right) + \sum_{t=2}^{T} \left(S_{t-1}  (\bx_t) -S_{t-1}  (\bx_{t-1}) \right)\\
	&=  \frac{1}{w}\sum_{t=T-w+1}^{T} f_t(\bx_t)  + \sum_{t=2}^{T} \left(S_{t-1}  (\bx_t) -S_{t-1}  (\bx_{t-1}) \right)\\
	&\leq M + \sum_{t=2}^{T} \left(S_{t-1}  (\bx_t) -S_{t-1}  (\bx_{t-1}) \right),
	\end{align*}
	where the last inequality follows from our blanket assumptions.
	Consequently, by Lemma \ref{lem:3}, we have that
	\begin{align*}
	S_{T} (\bx_T) + g(\bx_T) - g(\bx_1)&\leq M +  \sum_{t=2}^{T} \left(S_{t-1}  (\bx_t) + g(\bx_t) -S_{t-1}  (\bx_{t-1}) - g(\bx_{t-1}) \right) \\
	&\leq M - \sum_{t=1}^{T-1} \tau_t \left( \eta - \frac{\eta^2 L}{2}\right) \frac{\delta^2}{w^2}\\
	&\leq M  - \tau \left( \eta - \frac{\eta^2 L}{2}\right)  \frac{\delta^2}{w^2},
	\end{align*}
	where the last inequality uses $\tau = \sum_{t=1}^{T-1} \tau_t $.
	On the other hand, by our blanket assumptions,
	\begin{align*}
	S_{T} (\bx_T) 
	=  \frac{1}{w} \sum_{i=T-w+1}^{T} f_{i} (\bx_i)  &\geq - M.
	\end{align*}
	By combining both sides we obtain that
	\begin{align*}
	-M \leq g(\bx_1) - g(\bx_T ) +  M  - \tau \left( \eta - \frac{\eta^2 L}{2}\right)  \frac{\delta^2}{w^2},
	\end{align*}
	and the desired immediately follows from the nonnegativity of $g$:
	\begin{equation*}
	\tau \leq  \frac{g(\bx_1) - g(\bx_T ) + 2 M }{\left( \eta - \frac{\eta^2 L}{2}\right) \frac{\delta^2}{w^2}} \leq \frac{2 w^2 (g(\bx_1)  + 2 M)}{\left( 2 - \eta L\right) \eta \delta^2}.
	\hfill
	\qedhere
	\end{equation*}
\end{proof}

We conclude with the implication of our guarantees to the stochastic offline setting.
\begin{proof}[Proof of Corollary \ref{cor:2}]
	From the choice of $t_*$, Jensen's inequality,  and  \cref{thm:1}, we have that
	\begin{align*}
		\mathbb{E}_{t_*} \left(  \left\|\nabla f (\bx_{t_*}) \right\|^2 \right) &= \dfrac{1}{T-w} \sum_{t=w}^{T}   \left\|\mathbb{E} \left( \nabla f_t (\bx_{t})  \right) \right\|^2  \\
		&= \dfrac{1}{T-w} \sum_{t=w}^{T}   \left\|\mathbb{E} \left( \frac{1}{w}\sum_{i=t-w+1}^{t}\nabla f_i (\bx_{t})  \right) \right\|^2  \\
		& \leq \dfrac{1}{T-w} \sum_{t=w}^{T} \mathbb{E} \left(  \left\|\frac{1}{w}\sum_{i=t-w+1}^{t}\nabla f_i (\bx_{t}) \right\|^2 \right) \\
		&\leq \dfrac{1}{T-w} \mathbb{E} \left( \reg_w (T) \right) \\
		&\leq \frac{2}{(T-w) w^2} \left( T \delta^2  +  V_w [T] \right). 
	\end{align*}
Plugging the parameters' values $T=2w$, $w = \lceil\sqrt{\frac{2(\delta^2 + c )}{\varepsilon}} \rceil$, and $V_w[T] = c T$, we immediately obtain that
	\begin{align*}
		\mathbb{E} \left(  \left\|\nabla f (\bx_{t_*}) \right\|^2 \right) \leq \frac{2}{(T-w) w^2}  \left( \delta^2  T  + V_w [T]\right) &\leq  \frac{4}{ w^2}  \left( \delta^2 + c \right)  \leq \varepsilon.
	\end{align*}
Once again, by plugging the parameters' values we obtain from  \eqref{eq:21} in Theorem \ref{thm:3} that
	\begin{align*}
		\tau &\leq \frac{2 w^2 (g(\bx_1)  + 2 M)}{\left( 2 - \eta L\right) \eta \delta^2}
		\propto  O(\varepsilon^{-1}).
	\end{align*}
	Since for each prox-grad update the algorithm computes $w$ gradient samples (for each function sampled in the time-window), the SFO complexity is
	\begin{equation*}
	\tau w \propto  O(\varepsilon^{-3/2}).
	\qedhere
	\end{equation*}
\end{proof}

\section{Proofs of  \cref{sec:4}}
\label{sec:4a}
Before proceeding to the stochastic analysis, we make some notational conventions for the sake of readability: $S_t \equiv S_{t,w}$, $T (\bx ; \bd) \equiv T^{f,g}_{\eta} (\bx ; \bd)$, and $\mathcal{P} (\bx ; \bd) \equiv \mathcal{P}_{\eta}^{g} (\bx ; \bd) $.
Additionally, we set $\by_t^k = \by_t^{\tau_t}$ for all $k\geq \tau_t$; this means that $\by_t^k = \by_t^{k+1}$ if and only if $k\geq \tau_t $.\\

The forthcoming analysis of \cref{alg:4} requires delicate treatment of what is known, and what is not, at specific moments during the run. 
To avoid confusion, we state explicitly what is included in the algorithm's natural filtration at time $t\geq 1$ and at each inner iteration $k\geq 1$, thus extending on our original description. 
\begin{defin}[Filtration]
	\label{def:2a}
	For all $t\geq 1$, the filtration $\mathcal{F}_{t}$ includes all gradient feedback up to, but not including, the execution of step 2 at stage $t$.
	In particular, it includes $f_t$, $\bx_t$ and $\tilde{\nabla}S_{t-1} (\bx_t)$, but it does not include $\tilde{\nabla} f_t (\bx_t)$.
	
	For all $t\geq 1$ and all $k\geq 1$, the filtration $\mathcal{F}_{t,k}$ includes all gradient feedback up to, but not including, the execution of the $k$-th iteration of step 5(b) at time $t$.
	In particular, it contains $\mathcal{F}_t$, and includes $\by_t^{k}, G_t^k,$ and $\by_t^{k+1}$, but it does not include $\{ \tilde{\nabla} f_i (\by^{k+1}_t)\}_{i=t-w}^t$, $G_t^{k+1}$.
\end{defin}

We will utilize two trivial technical corollaries of   \cref{def:1} given next.
\begin{corollary}
	Let $\bx\in\real^n$, then 
	\begin{equation}
	\label{eq:17}
	\mathbb{E} \left(\| \mathcal{S}_{\sigma}(\bx;\omega,h)- \nabla h (\bx) \| \right)^2  \leq \mathbb{E} \left(\| \mathcal{S}_{\sigma}(\bx;\omega,h)- \nabla h (\bx) \|^2 \right) \leq  \sigma^2.
	\end{equation}
\end{corollary}
\begin{lem}
	\label{lem:5}
	Let $\bx\in\real^n$ and $h_i:\real^n\rightarrow\real$ for any $i=1,2,\ldots,w$.
	Then
	\begin{equation*}
	\mathbb{E} \left( \left\| \frac{1}{w}\sum_{i=1}^w \mathcal{S}_{\sigma}(\bx;\omega,h_i)-  \frac{1}{w}\sum_{i=1}^w \nabla h_i (\bx) \right\|^2 \right) \leq  \sigma^2.
	\end{equation*}
\end{lem}
\begin{proof}
	Follows from Jensen's inequality.
\end{proof}
The following technical lemma is of key importance in the analysis ahead.
\begin{lem}
	\label{lem:9}
	Let $t\in [T]$ and $k\geq 2$.
	It holds that
	\begin{equation*}
	\mathbb{E} \left( \langle G^k_{t} - \nabla S_{t} (\by_t^k), \by_t^{k+1} - \by_t^k \rangle | \mathcal{F}_{t,k-1} \right) \geq -\dfrac{\eta\sigma^2}{w^2}.
	\end{equation*}
\end{lem}
\begin{proof}
	Define the full gradient prox-grad  by $\hat{\by}^k_t = T^{g}_{\eta } (\by_t^k; \nabla S_{t,w} (\by_t^k) )$, and note that
	\begin{align}
	\langle G^k_{t} - \nabla S_{t} (\by_t^k), \by_t^{k+1} - \by_t^k \rangle  &= \langle G^k_{t} - \nabla S_{t} (\by_t^k), \by_t^{k+1} - \hat{\by}^k_t \rangle  + \langle G^k_{t} - \nabla S_{t} (\by_t^k), \hat{\by}^k_t- \by_t^k \rangle \tag*{}\\
	&\geq - \| G^k_{t} - \nabla S_{t} (\by_t^k) \| \|  \by_t^{k+1}  - \hat{\by}^k_t \|   +  \langle G^k_{t} - \nabla S_{t} (\by_t^k), \hat{\by}^k_t- \by_t^k \rangle, \label{eq:3a}
	\end{align}
	where the last inequality follows from Cauchy-Schwartz inequality.
	By the nonexpansivity of the prox operator \cite[Theorem 6.42]{B17} we have that
	\begin{equation*}
	\|  \by_t^{k+1}  - \hat{\by}^k_t \| \leq \|\by_t^k - \eta G^k_{t} - \by_t^k + \eta \nabla  S_{t} (\by_t^k) \| = \eta  \| G^k_{t} - \nabla S_{t} (\by_t^k) \|,
	\end{equation*}
	meaning that
	\begin{equation}
	\label{eq:2a}
	-\| G^k_{t} - \nabla S_{t} (\by_t^k) \|  \|  \by_t^{k+1}  - \hat{\by}^k_t \| \geq -\eta  \| G^k_{t} - \nabla S_{t} (\by_t^k) \|^2.
	\end{equation}
	Plugging (\ref{eq:2a}) to (\ref{eq:3a}) then implies that
	\begin{equation}
	\label{eq:4a}
	\langle G^k_{t} - \nabla S_{t} (\by_t^k), \by_t^{k+1} - \by_t^k \rangle \geq  -\eta  \| G^k_{t} - \nabla S_{t} (\by_t^k) \|^2 + \langle G^k_{t} - \nabla S_{t} (\by_t^k), \hat{\by}^k_t- \by_t^k \rangle.
	\end{equation}
	Noting that by Definition \ref{def:1}
	\begin{equation*}
	\mathbb{E} \left( \langle G^k_{t} - \nabla S_{t} (\by_t^k), \hat{\by}^k_t- \by_t^k \rangle | \mathcal{F}_{t,k-1}\right)   = 0,
	\end{equation*}
	we obtain, from taking expectation on (\ref{eq:4a}) and  using  Lemma \ref{lem:5}, that
	\begin{equation*}
	\mathbb{E} \left( \langle G^k_{t} - \nabla S_{t} (\by_t^k), \by_t^{k+1} - \by_t^k \rangle | \mathcal{F}_{t,k-1} \right) \geq - \dfrac{\eta\sigma^2}{w^2}.
	\qedhere
	\end{equation*}
\end{proof}

We can now embark on proving our claims  stated in  \cref{sec:4}.
\begin{proof}[Proof of Theorem \ref{thm:4}]
	Recall that $ \by_t^1 = \bx_t, \by_t^{\tau_t} = \bx_{t+1},$ and
	$$ \by_t^{k+1} = \argmin_{\bz\in\real^n} g(\bz) +\langle    G^k_{t}, \bz-\by_t^k \rangle   + \frac{1}{2\eta} \| \bz - \by_t^k\|^2, \qquad k\in [\tau_t-1].$$
	Denote $h^k_t := S_{t} (\by_t^k) + g (\by_t^k)$. By combining the descent lemma (cf. Lemma \ref{lem:des}), the definition of $\by_t^{k+1}$, and the stopping criteria of the inner loop, we have that for any $k\in [\tau_t-1]$ (assuming that $\mathcal{F}_{t}$ is given),
	\begin{align*}
	h^k_t - h^{k+1}_t &\geq \langle G^k_{t} - \nabla S_{t} (\by_t^k), \by_t^{k+1} - \by_t^k \rangle + \frac{1}{2}\left( \eta - \eta^2 L\right) \left\|  \mathcal{P} (\by_t^k ; G^k_{t}) \right\|^2 \\
	&\geq \langle G^k_{t} - \nabla S_{t} (\by_t^k), \by_t^{k+1} - \by_t^k \rangle + \frac{1}{2}\left( \eta - \eta^2 L\right) \dfrac{\delta^2}{w^2}. 
	\end{align*}
	Applying expectation to the latter, using the law of total expection (tower rule), and invoking Lemma \ref{lem:9} and relation \eqref{eq:23}, we obtain that for any $k\in [\tau_t-1]$ it holds that
	\begin{align*}
	\mathbb{E} \left( h^k_t - h^{k+1}_t \right) &\geq \mathbb{E} \left( \langle G^k_{t} - \nabla S_{t} (\by_t^k), \by_t^{k+1} - \by_t^k \rangle  \right) + \frac{1}{2}\left( \eta - \eta^2 L\right) \dfrac{\delta^2}{w^2}\\
	&\geq \dfrac{2}{w^2} \left( \eta\left(  1 - \eta L\right)\delta^2 - 2 \sigma^2 \right) > 0.
	\end{align*}
	Set $\alpha := 2\left( \eta\left(  1 - \eta L\right)\delta^2 - 2 \sigma^2 \right)/w^2 >0$.
	From the former, by using the law of total expectation, for any $K\geq 1$ we have that 
	\begin{align*}
	h^1_t + M \geq \mathbb{E} \left( h^1_t - h^{K+1}_t \right) 	 &= \mathbb{E} \left(\sum_{k=1}^{K}( h^k_t - h^{k+1}_t )\right) \\
		&= \sum_{k=1}^{K}\mathbb{E} \left( h^k_t - h^{k+1}_t \right) \\
	&= \sum_{k=1}^{K}\left(  \mathbb{E} \left( h^k_t - h^{k+1}_t | \tau_t\geq k+1 \right) \mathbb{P} (\tau_t\geq k+1)  + 0\cdot \mathbb{P} (\tau_t\leq k) \right) \\
		&	 \geq \alpha \sum_{k=1}^{K}  \mathbb{P} (\tau_t > k) \\
	&\geq \alpha \sum_{k=1}^{K}  \mathbb{P} (\tau_t > K) = \alpha  K  \mathbb{P} (\tau_t > K).
	\end{align*}
	Consequently, we must have that $\tau_t$ is almost surely finite, which in turn implies that $\tau$ must be almost surely finite as it is the finite sum of almost surely finite variables.
\end{proof}
Let us now establish the local regret bound stated in \cref{thm:2}.
\begin{proof}[Proof of \cref{thm:2}]
	Recall that
	\begin{equation}
	\label{eq:31a}
	\reg_w (T) = \sum_{t=1}^{T} \left\| \mathcal{P}(\bx_{t}; \nabla S_{t} (\bx_{t})) \right\|^2 = \sum_{t=1}^{T} \frac{1}{\eta^2}\left\| \bx_t - T (\bx_t; \nabla S_{t} (\bx_{t})) \right\|^2 .
	\end{equation}
	By simple algebra,
	\begin{align}
	\left\| \bx_t - T (\bx_t; \nabla S_{t} (\bx_{t})) \right\|^2 
	&\leq 2 \left\| \bx_t - T (\bx_t;\tilde{\nabla} S_{t} (\bx_{t})) \right\|^2 +2 \left\| T (\bx_t;\tilde{\nabla} S_{t} (\bx_{t})) - T (\bx_t; \nabla S_{t} (\bx_{t})) \right\|^2. \label{eq:32a}
	\end{align}
	Using the nonexpansivity of the prox operator \cite[Theorem 6.42]{B17} we have that
	\begin{align*}
	\left\| T (\bx_t;\tilde{\nabla} S_{t} (\bx_{t})) - T (\bx_t; \nabla S_{t} (\bx_{t})) \right\|^2
		&\leq \left\| \bx_t - \eta \tilde{\nabla} S_{t} (\bx_{t}) - \bx_t + \eta \nabla S_{t} (\bx_{t}) \right\|^2
	\\
		&= \eta^2 \left\| \tilde{\nabla} S_{t} (\bx_{t}) - \nabla S_{t} (\bx_{t}) \right\|^2.
	\end{align*}
	Subsequently, using the law of total expectation and \cref{lem:5}, we obtain the relation
	\begin{align*}
	\mathbb{E} \left( \left\| T (\bx_t;\tilde{\nabla} S_{t} (\bx_{t})) - T (\bx_t; \nabla S_{t} (\bx_{t})) \right\|^2 \right) &= \mathbb{E} \left[ \mathbb{E} \left( \left\| T (\bx_t;\tilde{\nabla} S_{t} (\bx_{t})) - T (\bx_t; \nabla S_{t} (\bx_{t})) \right\|^2 | \mathcal{F}_t \right)\right] \\
	&\leq \eta^2\mathbb{E} \left[ \mathbb{E} \left(\left\| \tilde{\nabla} S_{t} (\bx_{t}) - \nabla S_{t} (\bx_{t}) \right\|^2 | \mathcal{F}_t \right)\right] \leq \dfrac{\eta^2\sigma^2}{w^2}.
	\end{align*}
	Then, plugging the latter to the expected value of \eqref{eq:32a} yields
	\begin{align*}
	\mathbb{E} \left( \left\| \bx_t - T (\bx_t; \nabla S_{t} (\bx_{t}))\right\|^2\right) 
	&\leq 2 \eta^2 \mathbb{E} \left(\left\| \mathcal{P} (\bx_{t};\tilde{\nabla} S_{t} (\bx_{t}) ) \right\|^2\right) + \dfrac{2 \eta^2\sigma^2}{w^2}.
	\end{align*}
	Thus,
	\begin{equation}
	\label{eq:19a}
	\mathbb{E} \left( \reg_w (T) \right) \leq 2\sum_{t=1}^{T} \left[ \mathbb{E} \left( \left\| \mathcal{P} (\bx_{t};\tilde{\nabla} S_{t,w} (\bx_{t}) ) \right\|^2 \right) + \dfrac{\sigma^2}{w^2}\right] .
	\end{equation}
	
	Setting $G_1 = \tilde{\nabla}S_{t-1} (\bx_t) $, $G_2 = \frac{1}{w} (\tilde{\nabla} f_t (\bx_t) - \tilde{\nabla} f_{t-w} (\bx_t) )$, and applying \cref{lem:2} yields
	\begin{align}
	\left\| \mathcal{P} (\bx_t ; \tilde{\nabla} S_{t} (\bx_t)) \right\| = \left\| \mathcal{P} (\bx_t; G_1 + G_2) \right\| &\leq \left\| \mathcal{P} (\bx_t; \tilde{\nabla}S_{t-1} (\bx_t)) \right\| +\frac{1}{w} \left\| \tilde{\nabla} f_t (\bx_t) - \tilde{\nabla} f_{t-w} (\bx_t) \right\| \tag*{} \\
	&\leq \dfrac{\delta}{w} +\frac{1}{w} \left\| \tilde{\nabla} f_t (\bx_t) - \tilde{\nabla} f_{t-w} (\bx_t) \right\|,\label{eq:12a}
	\end{align}
	where the last inequality follows from the termination rule of the inner loop.
	Therefore, 
	\begin{align*}
	\left\| \mathcal{P} (\bx_t ; \tilde{\nabla} S_{t} (\bx_t)) \right\|^2 
	&\leq \frac{2}{w^2} \left( \delta^2 + \left\| \tilde{\nabla} f_t (\bx_t) - \tilde{\nabla} f_{t-w} (\bx_t) \right\|^2\right) .
	\end{align*}
	Using the triangle inequality and the relation $(a+b+c)^2\leq 3(a^2 + b^2 + c^2)$, yields that
	\begin{align*}
	&\left\| \tilde{\nabla} f_t (\bx_t) - \tilde{\nabla} f_{t-w} (\bx_t) \right\|^2 \leq \\ 
	&3\left\| \tilde{\nabla} f_t (\bx_t)- \nabla f_t(\bx_t) \right\|^2 + 3\left\| \nabla f_t(\bx_t) - \nabla f_{t-w}(\bx_t )\right\| ^2 + 3\left\| \nabla f_{t-w} (\bx_t ) - \tilde{\nabla} f_{t-w} (\bx_t) \right\| ^2.
	\end{align*}
	Applying expectation, from the law of total expectation together with \cref{def:1}, we obtain that
	\begin{equation*}
	\begin{array}{lll}
	\mathbb{E} \left[ \left\| \tilde{\nabla} f_t (\bx_t)- \nabla f_t(\bx_t) \right\|^2\right] &= \mathbb{E} \left[ \mathbb{E} \left( \left\| \tilde{\nabla} f_t (\bx_t)- \nabla f_t(\bx_t) \right\|^2 | \mathcal{F}_t\right) \right] &\leq \dfrac{\sigma^2}{w^2}, \\
	\mathbb{E} \left[ \left\| \nabla f_{t-w} (\bx_t ) - \tilde{\nabla} f_{t-w} (\bx_t) \right\|^2\right] &= \mathbb{E} \left[ \mathbb{E} \left( \left\| \nabla f_{t-w} (\bx_t ) - \tilde{\nabla} f_{t-w} (\bx_t) \right\|^2 | \mathcal{F}_{t-w}, \bx_t\right) \right] &\leq \dfrac{\sigma^2}{w^2}. 
	\end{array}
	\end{equation*}
	Thus, 
	$\mathbb{E} \left( \left\| \tilde{\nabla} f_t (\bx_t) - \tilde{\nabla} f_{t-w} (\bx_t) \right\|^2 \right) \leq \dfrac{6 \sigma^2}{w^2} + 3\mathbb{E} \left( \left\| \nabla f_t(\bx_t) - \nabla f_{t-w}(\bx_t )\right\| ^2\right)$, and consequently
	\begin{align*}
	\mathbb{E} \left( \left\| \mathcal{P} (\bx_t ; \tilde{\nabla} S_{t} (\bx_t)) \right\|^2 \right) &\leq\frac{2}{w^2} \left( \delta^2 + \mathbb{E} \left( \left\| \tilde{\nabla} f_t (\bx_t) - \tilde{\nabla} f_{t-w} (\bx_t) \right\|^2\right)\right) \\
	&\leq\frac{2}{w^2} \left( \delta^2 + \dfrac{6 \sigma^2}{w^2} + 3\mathbb{E} \left( \left\| \nabla f_t(\bx_t) - \nabla f_{t-w}(\bx_t )\right\| ^2\right)\right) .
	\end{align*}
	Summing over $t\in [T]$ and plugging $V_w [T]$ defined in \eqref{eq:V} then yields
	\begin{align*}
	\sum_{t=1}^{T} \mathbb{E} \left( \left\| \mathcal{P} (\bx_t ; \tilde{\nabla} S_{t} (\bx_t)) \right\|^2 \right) &\leq 2 \left( \delta^2 + \dfrac{6\sigma^2}{w^2} \right) \left( \frac{T}{w^2} \right) + \frac{6}{w^2} V_w [T] .
	\end{align*}
	Finally, plugging the latter into \eqref{eq:19a}, and recalling that $w\geq 1$, results with the desired bound.
\end{proof}

Finally, we prove the bound on the number of SFO calls, as stated by \cref{thm:5}.
\begin{proof}[Proof of \cref{thm:5}]
	Denote $h^k_t := S_{t} (\by_t^k) + g (\by_t^k)$. By combining the descent lemma (cf. \cref{lem:des}), the definition of the sequence $\{\by_t^k\}_{k\geq 1}$, Young's inequality, and the stopping criteria of the inner loop, we have that for any $K\geq 1$ (assuming that $\mathcal{F}_{t}$ is given)
	\begin{align*}
	h^1_t - h^{K+1}_t &= \sum_{k=1}^{K}( h^k_t - h^{k+1}_t )
	\geq \sum_{k=1}^{\min\{K,\tau_t\}} \left( \langle G^k_{t} - \nabla S_{t} (\by_t^k), \by_t^{k+1} - \by_t^k \rangle + \frac{ 1- \eta L}{2\eta} \left\| \by^{k+1}_t - \by^k_t \right\|^2 \right) \\
	&\geq \dfrac{1}{2}\sum_{k=1}^{\min\{K,\tau_t\}} \left( - \left\| G^k_{t} - \nabla S_{t} (\by_t^k)\right\|^2 - \left\| \by_t^{k+1} - \by_t^k \right\|^2 + \frac{1- \eta L}{\eta} \left\| \by^{k+1}_t - \by^k_t \right\|^2 \right).
	\end{align*}
	Hence, by \cref{ass:3} and the stopping condition of the inner loop, we obtain 
	\begin{align*}
	h^1_t - h^{K+1}_t 	&\geq 
	\dfrac{1}{2 w^2}\sum_{k=1}^{\min\{K,\tau_t\}} \left( -\sigma^2 + (1- \eta (L+1))\eta\delta^2 \right) 
	= \dfrac{ (1- \eta (L+1))\eta \delta^2 - \sigma^2}{2 w^2} \min\{K,\tau_t\} > 0.
	\end{align*}
	
	Recall that $S_{0,w} (\bx_0) \equiv 0$, and $S_{t} (\bx) = \frac{1}{w}(f_t(\bx) - f_{t-w}(\bx)) +S_{t-1} (\bx) $.
	Using the previous derivations for $t-1$ (setting $K = \tau_{t-1} $ and noting that $h^{\tau_{t-1}+1}_{t-1} = h^{\tau_{t-1}}_{t-1}$), we have that
	\begin{equation}
	\label{eq:24a}
	S_{t-1} (\bx_t) + g(\bx_t) -S_{t-1} (\bx_{t-1}) - g(\bx_{t-1}) = h^{\tau_{t-1}}_{t-1} - h^1_{t-1} \leq - \tau_{t-1} \dfrac{(1- \eta (L+1))\eta \delta^2 - \sigma^2}{2  w^2}.
	\end{equation}
	Thus, since 
	\begin{align*}
	S_{T} (\bx_T) 	= \sum_{t=1}^{T} (S_{t} (\bx_t)-S_{t-1} (\bx_{t-1}) ) 
	&= \sum_{t=1}^{T} \left( \frac{1}{w}(f_t(\bx_t) - f_{t-w}(\bx_t)) +S_{t-1} (\bx_t) -S_{t-1} (\bx_{t-1}) \right)  \\
	&= \frac{1}{w} \sum_{t=T-w+1}^{T} f_t(\bx_t)  + \sum_{t=2}^{T} \left(S_{t-1} (\bx_t) -S_{t-1} (\bx_{t-1}) \right),
	\end{align*}
	we have from our blanket assumptions and relation \eqref{eq:24a}, that 
	\begin{align*}
	S_{T} (\bx_T) 	
	&\leq g(\bx_1) - g(\bx_T) + M - \tau \dfrac{(1- \eta (L+1))\eta \delta^2 -\sigma^2}{2 w^2}.
	\end{align*}
	
	On the other hand, again by our blanket assumptions, $S_{T} (\bx_T) 	= \frac{1}{w} \sum_{i=T-w+1}^{T} f_{i} (\bx_i)  \geq -M.$
	By combining both sides, we obtain that
	\begin{align*}
	-M  \leq g(\bx_1) - g(\bx_T ) +  M - \tau \dfrac{(1- \eta (L+1))\eta \delta^2 -\sigma^2}{2 w^2},
	\end{align*}
	and the bound on $\tau$ immediately follows due to the nonnegativity of $g$.
	Finally, the desired bound on the SFO oracle calls follows from the fact that the inner loop makes $O(w)$ SFO calls per loop.
\end{proof}

\subsection{Implications to Offline Stochastic Optimization}
Next we establish our derivations in the offline scenario described in  \cref{sec:41}.
\begin{proof}[Proof of Theorem \ref{thm:6}]
	Note that $f_t \equiv f$ for any $t\in [T]$ implies that $\nabla S_{t,w} (\bx) \equiv \nabla f (\bx)$.
	From Theorem \ref{thm:2} and the choice of $t_*$ we have that
	\begin{align*}
	\mathbb{E} \left(  \left\| \mathcal{P}(\bx_{t_*}; \nabla f (\bx_{t_*})) \right\|^2 \right) &= \dfrac{1}{T-w} \mathbb{E} \left( \sum_{t=w}^{T} \left\| \mathcal{P}(\bx_{t}; \nabla f (\bx_{t})) \right\|^2 \right) \\
	&\leq \dfrac{1}{T-w} \mathbb{E} \left( \reg_w (T) \right) \\
	&\leq \frac{2}{(T-w) w^2}  \left( \left( \delta^2 + 7\sigma^2\right)T  + 6 V_w [T]\right) .
	\end{align*}
\end{proof}
\begin{proof}[Proof of Corollary \ref{cor:1}]
	From Theorem \ref{thm:6} we immediately obtain that
	\begin{align*}
	\frac{2}{(T-w) w^2}  \left( \left( \delta^2 + 7\sigma^2\right)T  + 6 V_w [T]\right) &=  \frac{4w}{ w^3}  \left( \delta^2 + 7\sigma^2 + c \right)  \leq \varepsilon.
	\end{align*}
	The bound $O(M\sigma\varepsilon^{-3/2}) $ is obtained by plugging the assumed values of $w,T$, and $\delta^2$, to  \eqref{eq:21} in Theorem \ref{thm:5}:
	\begin{align*}
	w\tau &\leq \frac{2\eta w^3 (g(\bx_1)  + 3M )}{(1- \eta (L+1))\delta^2 - \eta\sigma^2 }
	=\frac{2 w^3 (g(\bx_1)  + 3 M )}{\sigma^2 }
	\propto  O(M \sigma\varepsilon^{-3/2}) ,
	\end{align*}
where we used the fact that  $w$ is $O(\sigma/\sqrt{\varepsilon})$.
\end{proof}

\bibliographystyle{plainnat}
\bibliography{ONCVX_bib,Bibliography-PM}

\end{document}